\documentclass[twocolumn]{article}

\title{
Equivalence of Equilibrium Propagation\\
and Recurrent Backpropagation
}
\author{
  Benjamin Scellier and Yoshua Bengio\footnote{Y.B. is also a Senior Fellow of CIFAR}\\
  MILA, Universit\'{e} de Montr\'{e}al
}

\usepackage[T1]{fontenc}      
\usepackage[utf8]{inputenc} 
\usepackage[english]{babel}

\usepackage{amssymb,amsthm,amsfonts,amscd}
\usepackage{amsmath}
\usepackage{mathtools}   
\usepackage{times}
\usepackage{textcomp}
\usepackage{fancyhdr}    

\usepackage[margin=3cm]{caption}
\usepackage{makeidx} 
\usepackage{graphicx}
\usepackage{algorithm}
\usepackage{algorithmic}
\usepackage{natbib}
\usepackage{caption}
\graphicspath{{.},{img/}}

\newcommand \x{\mathrm x}
\newcommand \y{\mathrm y}
\newcommand \s{\mathrm s}

\newcommand{\norm}[1]{\left\lVert #1\right \rVert} 

\newtheorem{thm}{Theorem}

\makeindex

\usepackage{vmargin}
\setmarginsrb{2cm}{2cm}{2cm}{3cm}{0cm}{0.8cm}{0cm}{1cm}

\begin{document}

\maketitle

\abstract{
Recurrent Backpropagation and Equilibrium Propagation are supervised learning algorithms for fixed point recurrent neural networks
which differ in their second phase.
In the first phase, both algorithms converge to a fixed point which corresponds to the configuration where the prediction is made.
In the second phase, Equilibrium Propagation relaxes to another nearby fixed point corresponding to smaller prediction error,
whereas Recurrent Backpropagation uses a side network to compute error derivatives iteratively.

In this work we establish a close connection between these two algorithms.
We show that, at every moment in the second phase, the temporal derivatives of the neural activities in Equilibrium Propagation
are equal to the error derivatives computed iteratively by Recurrent Backpropagation in the side network.
This work shows that it is not required to have a side network for the computation of error derivatives,
and supports the hypothesis that, in biological neural networks, temporal derivatives of neural activities may code for error signals.
}


\section{Introduction}

In Deep Learning, the backpropagation algorithm used to train neural networks requires a side network for the propagation of error derivatives,
which is widely seen as biologically implausible \citep{crick-nature1989}.
One fascinating hypothesis, first formulated by \citet{Hinton+McClelland-1988}, is that,
in biological neural networks, error signals could be encoded in the temporal derivatives of the neural activities.
This allows for error signals to be propagated in the network via the neuronal dynamics itself without the need for a side network.
Neural computation would correspond to both inference and error back-propagation.
The work presented in this paper also supports this hypothesis.

In section \ref{sec:prior-work}, we present the machine learning setting we are interested in.
The neurons of the network follow the gradient of an energy function, such as the Hopfield energy \citep{cohen1983absolute,Hopfield84}.
Energy minima correspond to preferred states of the model.
At prediction time, inputs are clamped and
the network relaxes to a fixed point, corresponding to a local minimum of the energy function.
The prediction is then read out on the output neurons.
This corresponds to the first phase of the algorithm.
The goal of learning is that of minimizing the cost at the fixed point, called objective.

Section \ref{sec:rec-backprop} presents \textit{Recurrent Backpropagation} \citep{Almeida87,Pineda87},
an algorithm which computes the gradient of the objective.
In the second phase of Recurrent Backpropagation, an iterative procedure computes error derivatives.

In section \ref{sec:equi-prop} we present \textit{Equilibrium Propagation} \citep{Scellier+Bengio-frontiers2017},
another algorithm which computes the gradient of the objective.
In the second phase of Equilibrium Propagation
when the target values for output neurons are observed,
the output neurons are nudged
towards their targets and the network starts a second relaxation phase towards a second but nearby
fixed point which corresponds to slightly smaller prediction error.
The gradient of the objective
can be computed based on a contrastive Hebbian learning rule at the first fixed point and second fixed point.

Section \ref{sec:coding} (in particular Theorem \ref{thm:error-derivatives}) constitutes the main contribution of our work.
We establish a close connection between Recurrent Backpropagation and Equilibrium Propagation.
We show that at every moment in the second phase of Equilibrium Propagation,
the temporal derivative of the neural activities \textit{code} (i.e. are equal to)
intermediate error derivatives which Recurrent Backpropagation computes iteratively.
Our work shows that one does not require a special computational path for the computation of the error derivatives in the second phase -
the same information is available in the temporal derivatives of the neural activities.
Furthermore we show that, in Equilibrium Propagation, halting the second phase before convergence to the second fixed point is equivalent to \textit{Truncated Recurrent Backpropagation}.


\section{Machine Learning Setting}
\label{sec:prior-work}

We consider the supervised setting in which we want to predict a \textit{target} $\y$ given an \textit{input} $\x$.
The pair $(\x,\y)$ is a data point.
The model is a network specified by a \textit{state variable} $s$ and a \textit{parameter variable} $\theta$.
The dynamics of the network are determined by two differentiable scalar functions $E_\theta(\x,s)$ and $C_\theta(\y,s)$
which we call \textit{energy function} and \textit{cost function} respectively.
In most of the paper, to simplify the notations we omit the dependence on $\x$ and $\y$ and simply write $E_\theta(s)$ and $C_\theta(s)$.
Furthermore we write $\frac{\partial E_\theta}{\partial \theta}(s)$ and $\frac{\partial E_\theta}{\partial s}(s)$ the partial derivatives of $(\theta,s) \mapsto E_\theta(s)$ with respect to $\theta$ and $s$, respectively.
Similarly  $\frac{\partial C_\theta}{\partial \theta}(s)$ and $\frac{\partial C_\theta}{\partial s}(s)$ denote the partial derivatives of $(\theta,s) \mapsto C_\theta(s)$.

The state variable $s$ is assumed to move spontaneously towards low-energy configurations by following the gradient of the energy function:
\begin{equation}
	\label{eq:free-dynamics}
	\frac{ds}{dt} = -\frac{\partial E_\theta}{\partial s}(s).
\end{equation}
The state $s$ eventually settles to a minimum of the energy function, written $s_\theta^0$ and characterized by
\footnote{In general, the fixed point defined by Eq.~\ref{eq:free-fixed-point} is not unique, unless further assumptions are made on $E_\theta(s)$ (e.g. convexity). The fixed point depends on the initial state of the dynamics (Eq.~\ref{eq:free-dynamics}), and so does the objective function of Eq.~\ref{eq:objective}. However, for ease of presentation, we shall avoid delving into these mathematical details here.
}
\begin{equation}
	\label{eq:free-fixed-point}
	\frac{\partial E_\theta}{\partial s} \left( s_\theta^0 \right) = 0.
\end{equation}
Since the dynamics in Eq.~\ref{eq:free-dynamics} only depends on the input $\x$ (through $E_\theta(\x,s)$) but not on the target $\y$,
we call this relaxation phase the \textit{free phase},
and the energy minimum $s_\theta^0$ is called \textit{free fixed point}.

The goal of learning is that of finding $\theta$ such that the cost at the fixed point $C_\theta \left( s_\theta^0 \right)$ is minimal.
\footnote{In this expression, both the cost function $C_\theta(s)$ and the fixed point $s_\theta^0$ depend on $\theta$.
$C_\theta(s)$ directly depends on $\theta$, whereas $s_\theta^0$ indirectly depends on $\theta$ through $E_\theta(s)$ (see Eq.~\ref{eq:free-fixed-point}).}
We introduce the \textit{objective function} (for a single data point $(\x,\y)$)
\begin{equation}
	\label{eq:objective}
	J(\theta) := C_\theta \left( s_\theta^0 \right).
\end{equation}
Note the distinction between the cost function and the objective function:
the cost function $C_\theta(s)$ is defined for any state $s$ whereas the objective function $J(\theta)$ is the cost at the fixed point.

Several methods have been proposed to compute the gradient of $J$ with respect to $\theta$.
Early work by \citet{Almeida87,Pineda87} have introduced an algorithm called Recurrent Backpropagation,
which we present in section \ref{sec:rec-backprop}.
In \citet{Scellier+Bengio-frontiers2017} we proposed another algorithm - at first sight very different.
We present it in section \ref{sec:equi-prop}.
In section \ref{sec:coding} we will show that there is actually a profound connection between these two algorithms.


\subsection{Example: Hopfield Model}
\label{sec:hopfield}

In this subsection we propose particular forms for the energy function $E_\theta(\x,s)$ and the cost function $C_\theta(\y,s)$ to ease understanding.
Nevertheless, the theory presented in this paper is general and does not rely on the particular forms of the functions $E_\theta(\x,s)$ and $C_\theta(\y,s)$ chosen here.



Recall that we consider the supervised setting where we must predict a target $\y$ given an input $\x$.
To illustrate the idea we consider the case where the neurons of the network are split in layers $s_0$, $s_1$ and $s_2$, as in Figure \ref{fig:network}.
\footnote{We choose to number the layers in increasing order from output to input, in the sense of propagation of error signals (see section \ref{sec:equi-prop}).}
In this setting, the state variable $s$ is the set of layers $s = \left( s_0,s_1,s_2 \right)$.
Each of the layers of neurons $s_0$, $s_1$ and $s_2$ is a vector whose coordinates are real numbers representing the membrane voltages of the neurons.
The output layer $s_0$ corresponds to the layer where the prediction is read and has the same dimension as the target $\y$.
Furthermore $\rho$ is a deterministic function (nonlinear activation) which maps a neuron's voltage onto its firing rate.
We commit a small abuse of notation and denote $\rho(s_i)$ the vector of firing rates of the neurons in layer $s_i$ ;
here the function $\rho$ is applied elementwise to the coordinates of the vector $s_i$.
Therefore the vector $\rho(s_i)$ has the same dimension as $s_i$.
Finally, the parameter variable $\theta$ is the set of (bidirectional) weight matrices between the layers $\theta = \left( W_{01},W_{12},W_{23} \right)$.

We consider the following modified Hopfield energy function:
\begin{align}
	\label{eq:hopfield-energy}
	\begin{split}
	E_\theta(\x,s) = & \frac{1}{2} \left( \norm{s_0}^2 + \norm{s_1}^2 + \norm{s_2}^2 \right) \\
	                 & - \rho(s_0)^T \cdot W_{01} \cdot \rho(s_1) \\
	                 & - \rho(s_1)^T \cdot W_{12} \cdot \rho(s_2) \\
	                 & - \rho(s_2)^T \cdot W_{23} \cdot \rho(\x).
	\end{split}
\end{align}
With this choice of energy function, the dynamics (Eq.~\ref{eq:free-dynamics}) translate into a form of leaky integration neural dynamics with symmetric connections:
\begin{align}
	\frac{ds_0}{dt} & = \rho'(s_0)^T \otimes W_{01} \cdot \rho(s_1) - s_0,                          \\
	\frac{ds_1}{dt} & = \rho'(s_1)^T \otimes \left( W_{12} \cdot \rho(s_2) + W_{01}^T \cdot \rho(s_0) \right) - s_1, \\
	\frac{ds_2}{dt} & = \rho'(s_2)^T \otimes \left( W_{23} \cdot \rho(\x ) + W_{12}^T \cdot \rho(s_1) \right) - s_2.
\end{align}
Here again the derivative of the function $\rho$ (denoted $\rho'$) is applied elementwise to the coordinates of the vectors $s_0$, $s_1$ and $s_2$,
and the notation $\otimes$ is used to mean element-wise multiplication.
\footnote{Given two vectors $a = \left( a_1,\ldots,a_n \right)$ and $b = \left( b_1,\ldots,b_n \right)$,
their product element by element is $a \otimes b = \left( a_1 b_1,\ldots,a_n b_n \right)$.}

\begin{figure}[h]
	\centering
	\captionsetup{width=.8\linewidth}
	\includegraphics[width=.6\linewidth]{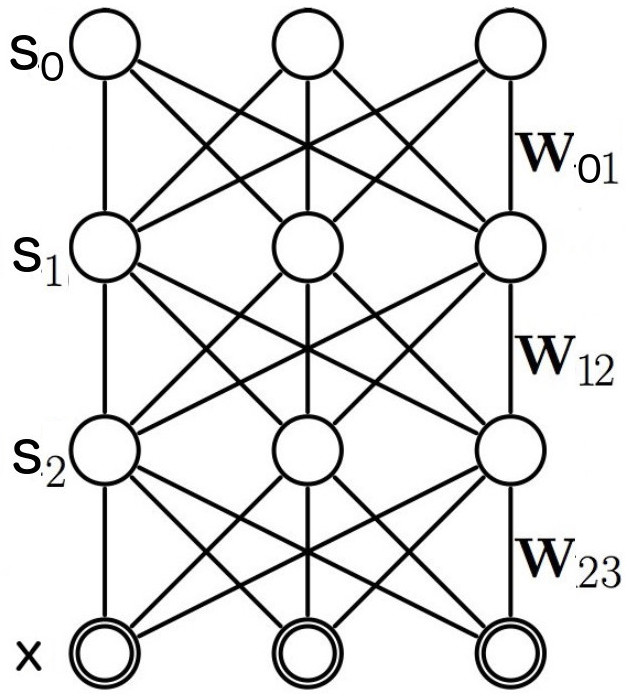}
	\caption{Graph of the network.
	Input $\x$ is clamped.
	State variable $s$ includes hidden layers $s_2$ and $s_1$,
	and output layer $s_0$ (layer where the prediction is read).
	Output layer $s_0$ has the same dimension as target $\y$.}
	\label{fig:network}
\end{figure}

Finally we consider the quadratic cost function
\begin{equation}
	\label{eq:quadratic-cost}
	C_\theta(\y,s) = \frac{1}{2} \norm{\y-s_0}^2,
\end{equation}
which measures the discrepancy between the output layer $s_0$ and the target $\y$.

Note that the results established in this paper hold for any energy function $E_\theta(s)$ and any cost function $C_\theta(s)$,
and are not limited to the Hopfield energy and the quadratic cost (Eq.~\ref{eq:hopfield-energy} and Eq.~\ref{eq:quadratic-cost}).


\section{Recurrent Back-Propagation}
\label{sec:rec-backprop}

In this section we present \textit{Recurrent Backpropagation}, an algorithm introduced by \citet{Almeida87,Pineda87}
which computes the gradient of $J$ (Eq.~\ref{eq:objective}).
The original algorithm was described in the discrete-time setting and for a general state-to-state dynamics.
Here we present it in the continuous-time setting in the particular case of a gradient dynamics (Eq.~\ref{eq:free-dynamics}).
A direct derivation based on the adjoint method can also be found in \citet{lecun1988theoretical}.


\subsection{Projected Cost Function}

Let $S_\theta^0(\s,t)$ denote the state of the network at time $t \geq 0$
when it starts from an initial state $\s$ at time $t=0$ and follows the free dynamics Eq.~\ref{eq:free-dynamics}.
In the theory of dynamical systems $S_\theta^0(\s,t)$ is called the \textit{flow}.
We introduce the \textit{projected cost function}
\begin{equation}
	\label{eq:projected-cost}
	L_\theta(\s,t) := C_\theta \left( S_\theta^0(\s,t) \right).
\end{equation}
This is the cost of the state projected a duration $t$ in the future, when the networks starts from $\s$ and follows the free dynamics.
Two particular cases of importance:
\begin{itemize}
	\item for $t=0$, the projected cost is simply the cost of the current state $L_\theta(\s,0) = C_\theta \left( \s \right)$,
	\item as $t \to \infty$ the projected cost converges to the objective $L_\theta(\s,t) \to J(\theta)$.
\end{itemize}
The second property comes from the fact that the dynamics converges to the fixed point $S_\theta^0(\s,t) \to s_\theta^0$ as $t \to \infty$.
More generally, for fixed $\s$ the process $\left( L_\theta(\s,t) \right)_{t \geq 0}$ represents the successive cost values taken by the state of the network along the free dynamics when it starts from the initial state $\s$.
Under mild regularity conditions on $E_\theta(s)$ and $C_\theta(s)$,
the gradient of the projected cost function converges to the gradient of the objective function in the limit of infinite duration:
\begin{equation}
	\label{eq:gradient-projected-cost}
	\frac{\partial L_\theta}{\partial \theta}(\s,t) \to \frac{\partial J}{\partial \theta}(\theta)
\end{equation}
as $t \to \infty$.
Therefore, if we can compute $\frac{\partial L_\theta}{\partial \theta}(\s,t)$ for a particular value of $\s$ and for any $t \geq 0$,
we can obtain the desired gradient $\frac{\partial J}{\partial \theta}(\theta)$ by letting $t \to \infty$.
We will show next
that this is what Recurrent Backpropagation does in the case where the initial state $\s$ is the fixed point $s_\theta^0$.


\subsection{Process of Error Derivatives}

We introduce the \textit{process of error derivatives} $(\overline{S}_t,\overline{\Theta}_t)_{t \geq 0}$, defined as
\begin{align}
	\overline{S}_t      & := \frac{\partial L_\theta}{\partial s}      \left( s_\theta^0,t \right), \qquad t \geq 0, \label{eq:error-derivative} \\
	\overline{\Theta}_t & := \frac{\partial L_\theta}{\partial \theta} \left( s_\theta^0,t \right), \qquad t \geq 0.
\end{align}
The process $\overline{S}_t$      takes values in the state     space (space of the state variable $s$) and
the process $\overline{\Theta}_t$ takes values in the parameter space (space of the parameter variable $\theta$).
\footnote{The quantity $\overline{\Theta}_t = \frac{\partial L_\theta}{\partial \theta}\left( s_\theta^0,t \right)$ represents the partial derivative of $L_\theta(\s,t)$ with respect to $\theta$, evaluated at the fixed point $\s=s_\theta^0$.
This does not include the differentiation path through the fixed point $s_\theta^0$}
As shown in Theorem \ref{thm:rec-backprop} below, the Recurrent Backpropagation algorithm computes $\overline{S}_t$ and $\overline{\Theta}_t$ iteratively for increasing values of $t$.
As a consequence of Eq.~\ref{eq:gradient-projected-cost}, we obtain the desired gradient $\frac{\partial J}{\partial \theta}(\theta)$ in the limit $t \to \infty$, i.e.
\begin{equation}
	\label{eq:limit}
	\overline{\Theta}_t \to \frac{\partial J}{\partial \theta}(\theta).
\end{equation}

\begin{thm}[Recurrent Backpropagation]
	\label{thm:rec-backprop}
	The process of error derivatives $(\overline{S}_t,\overline{\Theta}_t)$ satisfies
	\begin{align}
		\label{eq:Cauchy-1}
		\overline{S}_0 & = \frac{\partial C_\theta}{\partial s} \left( s_\theta^0 \right),      \\
		\label{eq:Cauchy-2}
		\overline{\Theta}_0 & = \frac{\partial C_\theta}{\partial \theta} \left( s_\theta^0 \right), \\
		\label{eq:Cauchy-3}
		\frac{d}{dt} \overline{S}_t & = - \frac{\partial^2 E_\theta}{\partial s^2} \left( s_\theta^0 \right) \cdot \overline{S}_t, \\
		\label{eq:Cauchy-4}
		\frac{d}{dt} \overline{\Theta}_t & = - \frac{\partial^2 E_\theta}{\partial \theta \partial s} \left( s_\theta^0 \right) \cdot \overline{S}_t.
	\end{align}
\end{thm}
Theorem \ref{thm:rec-backprop} is proved in Appendix \ref{sec:proofs}.
Theorem \ref{thm:rec-backprop} offers a way to compute the gradient $\frac{\partial J}{\partial \theta} \left( \theta \right)$.
In the first phase (or free phase), the state variable $s$ follows the free dynamics (Eq.~\ref{eq:free-dynamics}) and relaxes to the fixed point $s_\theta^0$.
Reaching this fixed point is necessary for evaluating the Hessian $\frac{\partial^2 E_\theta}{\partial s^2} \left( s_\theta^0 \right)$
which is required in the second phase.
In the second phase, one computes $\overline{S}_t$ and $\overline{\Theta}_t$ iteratively for increasing values of $t$ thanks to
Eq.~\ref{eq:Cauchy-1}, Eq.~\ref{eq:Cauchy-2}, Eq.~\ref{eq:Cauchy-3} and Eq.~\ref{eq:Cauchy-4}.
We obtain the desired gradient as $t \to \infty$ (Eq.~\ref{eq:limit}).

Note that the Hessian $\frac{\partial^2 E_\theta}{\partial s^2} \left( s_\theta^0 \right)$ is positive definite
since $s_\theta^0$ is an energy minimum.
Therefore Eq.~\ref{eq:Cauchy-3} guarantees that $\frac{\partial L_\theta}{\partial s} \left( s_\theta^0,t \right) \to 0$ as $t \to \infty$,
in agreement with the fact that $J(\theta)$ is (locally) insensitive to the initial state ($\s = s_\theta^0$ in our case).

From the point of view of biological plausibility,
the requirement to run the dynamics for $\overline{S}_t$ and $\overline{\Theta}_t$ to compute the gradient $\frac{\partial J}{\partial \theta}(\theta)$ is not satisfying.
It is not clear what the quantities $\overline{S}_t$ and $\overline{\Theta}_t$ would represent in a biological network.
This issue is adressed in sections \ref{sec:equi-prop} and \ref{sec:coding}.


\section{Equilibrium Propagation}
\label{sec:equi-prop}

In this section, we present Equilibrium Propagation \citep{Scellier+Bengio-frontiers2017},
another algorithm which computes the gradient of the objective function $J$ (Eq.~\ref{eq:objective}).
At first sight, Equilibrium Propagation and Recurrent Backpropagation share little in common.
However in section \ref{sec:coding} we will show a profound connection between these algorithms.


\subsection{Augmented Energy Function}

The central idea of Equilibrium Propagation is to introduce the \textit{augmented energy function}
\begin{equation}
	E_\theta^\beta(s) := E_\theta(s) + \beta \; C_\theta(s),
\end{equation}
where $\beta \geq 0$ is a scalar which we call \textit{influence parameter}.
The free dynamics (Eq.~\ref{eq:free-dynamics}) is then replaced by the \textit{augmented dynamics}
\begin{equation}
	\label{eq:augmented-dynamics}
	\frac{ds}{dt} = -\frac{\partial E_\theta^\beta}{\partial s}(s).
\end{equation}
The state variable now follows the dynamics $\frac{ds}{dt} = - \frac{\partial E_\theta}{\partial s}(s) - \beta \frac{\partial C_\theta}{\partial s}(s)$.
When $\beta > 0$, in addition to the usual term $-\frac{\partial E_\theta}{\partial s}(s)$,
an additional term $-\beta\frac{\partial C_\theta}{\partial s}(s)$ nudges $s$ towards configurations that have lower cost values.
In the case of the model described in section \ref{sec:hopfield} with the quadractic cost function (Eq.~\ref{eq:quadratic-cost}),
the new term $-\beta\frac{\partial C_\theta}{\partial s}(s)$ is the vector of $\dim(s)$ whose component on $s_0$ is $\beta \left( \y-s_0 \right)$
and whose components on $s_1$ and $s_2$ are zero.
Thus, the new term takes the form of a 'force' that nudges the output layer $s_0$ towards the target $\y$.
Unlike the free dynamics which only depends on $\x$ (through $E_\theta(\x,s)$) but not on $\y$,
the augmented dynamics also depends on $\y$ (through $C_\theta(\y,s)$).

Note that the free dynamics corresponds to the value $\beta=0$.
We then generalize the notion of fixed point for any value of $\beta$.
The augmented dynamics converges to the fixed point $s_\theta^\beta$, an energy minimum of $E_\theta^\beta$ characterized by
\begin{equation}
	\frac{\partial E_\theta^\beta}{\partial s} \left( s_\theta^\beta \right) = 0.
\end{equation}

Theorem \ref{thm:equilibrium-prop} below shows that the gradient $\frac{\partial J}{\partial \theta}(\theta)$ can be estimated based on measures at the fixed points $s_\theta^0$ and $s_\theta^\beta$.

\begin{thm}[Equilibrium Propagation]
	\label{thm:equilibrium-prop}
	The gradient of the objective function with respect to $\theta$ can be estimated thanks to the formula
	\begin{equation}
		\label{eq:gradient-formula}
		\frac{\partial J}{\partial \theta}(\theta) =
		\lim_{\beta \to 0} \frac{1}{\beta} \left( \frac{\partial E_\theta^\beta}{\partial \theta} \left( s_\theta^\beta \right)
		- \frac{\partial E_\theta^0}{\partial \theta} \left( s_\theta^0 \right) \right).
	\end{equation}
\end{thm}
A proof of Theorem \ref{thm:equilibrium-prop} is given in Appendix \ref{sec:equi-prop-proof}.
Note that Theorem \ref{thm:equilibrium-prop} is also a consequence of the more general formula of Eq.~\ref{eq:cor-2} (Theorem \ref{thm:error-derivatives}),
established in the next section.

Theorem \ref{thm:equilibrium-prop} offers another way to estimate the gradient of $J(\theta)$.
As in Recurrent Backpropagation, in the first phase (or `free phase'), the network follows the free dynamics (Eq.~\ref{eq:free-dynamics}).
This is equivalent to saying that the network follows the augmented dynamics (Eq.~\ref{eq:augmented-dynamics}) when the value of $\beta$ is set to $0$.
The network relaxes to the free fixed point $s_\theta^0$, where $\frac{\partial E_\theta}{\partial \theta} \left( s_\theta^0 \right)$ is measured.
In the second phase which we call \textit{nudged phase}, the influence parameter takes on a small positive value $\beta \gtrsim 0$,
and the network relaxes to a new but nearby fixed point $s_\theta^\beta$ where $\frac{\partial E_\theta^\beta}{\partial \theta} \left( s_\theta^\beta \right)$ is measured.
The gradient of the objective function is estimated thanks to the formula in Eq.~\ref{eq:gradient-formula}.

In the case of the modified Hopfield energy (Eq.~\ref{eq:hopfield-energy}) the components of $\frac{\partial E_\theta}{\partial \theta}(s)$ are $\frac{\partial E_\theta}{\partial W_{01}}(s)$, $\frac{\partial E_\theta}{\partial W_{12}}(s)$ and $\frac{\partial E_\theta}{\partial W_{23}}(s)$.
For instance
$\frac{\partial E_\theta}{\partial W_{01}}(s) = - \rho \left( s_0 \right) \cdot \rho \left( s_1 \right)^T$ is a matrix of size $\dim(s_0) \times \dim(s_1)$
whose entries can be measured locally at each synapse based on the presynaptic activity and postsynaptic activity.
Thus the learning rule of Eq.~\ref{eq:gradient-formula} is a kind of contrastive Hebbian learning rule at the free and nudged fixed points.

At the beginning of the second phase, the network is initially at the free fixed point $s_\theta^0$.
When the influence parameter takes on a small positive value $\beta \gtrsim 0$,
the novel term $-\beta \frac{\partial C_\theta}{\partial s}(s)$ in the dynamics of the state variable perturbs the system.
This perturbation propagates in the layers of the network until convergence to the new fixed point $s_\theta^\beta$.

In the next section, we go beyond the analysis of fixed points and we show that, at every moment $t$ in the nudged phase,
the temporal derivative $\frac{ds}{dt}$ encodes the error derivative of Eq.~\ref{eq:error-derivative}.


\section{Temporal Derivatives Code for Error Derivatives}
\label{sec:coding}

Theorem \ref{thm:equilibrium-prop} shows that the gradient of $J$ can be estimated based on the free and nudged fixed points only.
In this section we are interested in the dynamics of the network in the second phase, from the free fixed point to the nudged fixed point.
Recall that $S_\theta^0(\s,t)$ is the flow of the dynamical system (Eq.~\ref{eq:free-dynamics}), that is the state of the network at time $t \geq 0$
when it starts from an initial state $\s$ at time $t=0$ and follows the free dynamics.
Similarly we define $S_\theta^\beta(\s,t)$ for any value of $\beta$ when the network follows the augmented dynamics (Eq.~\ref{eq:augmented-dynamics}).

In Equilibrium Propagation, the state of the network at the beginning of the nudged phase is the free fixed point $s_\theta^0$.
We choose as origin of time $t=0$ the moment when the second phase starts:
the network is in the state $s_\theta^0$
and the influence parameter takes on a small positive value $\beta \gtrsim 0$.
With our notations, the state of the network after a duration $t$ in the nudged phase is $S_\theta^\beta \left( s_\theta^0,t \right)$.
As $t \to \infty$ the network's state converges to the nudged fixed point $S_\theta^\beta \left( s_\theta^0,t \right) \to s_\theta^\beta$.


\subsection{Process of Temporal Derivatives}

Now we are ready to introduce the \textit{process of temporal derivatives} $(\widetilde{S}_t,\widetilde{\Theta}_t)_{t \geq 0}$, defined by
\begin{align}
	\label{eq:temporal-derivative-s}
	\widetilde{S}_t      & := - \lim_{\beta \to 0} \frac{1}{\beta} \frac{\partial S_\theta^\beta}{\partial t} \left( s_\theta^0,t \right), \\
	\label{eq:temporal-derivative-theta}
	\widetilde{\Theta}_t & := \lim_{\beta \to 0} \frac{1}{\beta} \left(
	\frac{\partial E_\theta^\beta}{\partial \theta} \left( S_\theta^\beta \left( s_\theta^0,t \right) \right)
	- \frac{\partial E_\theta^0}{\partial \theta} \left( s_\theta^0 \right)
	\right).
\end{align}
Like $\overline{S}_t$ and $\overline{\Theta}_t$, the processes $\widetilde{S}_t$ and $\widetilde{\Theta}_t$ take values in the state space and parameter space respectively.

The process $\widetilde{S}_t$ is simply the temporal derivative $\frac{ds}{dt}$ in the second phase,
rescaled by $\frac{1}{\beta}$ (so that its value does not depend on the particular choice of $\beta \gtrsim 0$).

Theorem \ref{thm:error-derivatives} below states that the processes $(\overline{S}_t,\overline{\Theta}_t)$ and $(\widetilde{S}_t,\widetilde{\Theta}_t)$ are equal.

\begin{thm}[Temporal Derivatives as Error Derivatives]
	\label{thm:error-derivatives}
	The processes of error derivatives and temporal derivatives are equal,
	that is $\overline{S}_t = \widetilde{S}_t$ and $\overline{\Theta}_t = \widetilde{\Theta}_t$ for every $t \geq 0$.
	Using explicit forms, this translates into:
	\small
	\begin{align}
		\label{eq:cor-1}
		\hspace*{-0.2cm}
		\frac{\partial L_\theta}{\partial s} \left( s_\theta^0,t \right) & =
		- \lim_{\beta \to 0} \frac{1}{\beta} \frac{\partial S_\theta^\beta}{\partial t} \left( s_\theta^0,t \right), \\
		\label{eq:cor-2}
		\hspace*{-0.2cm}
		\frac{\partial L_\theta}{\partial \theta} \left( s_\theta^0,t \right) & = \lim_{\beta \to 0} \frac{1}{\beta} \left(
		\frac{\partial E_\theta^\beta}{\partial \theta} \left( S_\theta^\beta \left( s_\theta^0,t \right) \right) - \frac{\partial E_\theta^0}{\partial \theta} \left( s_\theta^0 \right)
		\right).
	\end{align}
	\normalsize
\end{thm}
Theorem \ref{thm:error-derivatives} is proved in Appendix \ref{sec:proofs}.
In essence, Eq.~\ref{eq:cor-1} says that in the second phase of Equilibrium Propagation,
the temporal derivative $\frac{ds}{dt}$ (rescaled by $\frac{1}{\beta}$) \textit{encodes} the error derivative (Eq.~\ref{eq:error-derivative}).

Here is an interpretation of Eq.~\ref{eq:cor-1}.
Suppose that the network is initially at the fixed point $\s = s_\theta^0$.
Consider the cost $L_\theta \left( s_\theta^0+\Delta s,t \right)$ a duration $t$ in the future if one moved the initial state $\s = s_\theta^0$ by a small step $\Delta s$.
The goal is to find the direction $\Delta s$ which would minimize $L_\theta \left( s_\theta^0+\Delta s,t \right)$.
The naive approach by trials and errors is neither biologically plausible nor efficient.
Eq.~\ref{eq:cor-1} tells us that there is a physically realistic way of finding such a direction $\Delta s$ in one attempt - this direction
is encoded in the temporal derivative $\frac{ds}{dt}$ at time $t$ after starting the nudged phase.

Note that as $t \to \infty$ both sides of Eq.~\ref{eq:cor-1} converge to $0$.
This is a consequence of Eq.~\ref{eq:Cauchy-3} and the fact that the Hessian $\frac{\partial^2 E_\theta}{\partial s^2} \left( s_\theta^0 \right)$ is positive definite (since $s_\theta^0$ is an energy minimum), as already mentioned in section \ref{sec:rec-backprop}.
Intuitively the right hand side converges to $0$ because $S_\theta^\beta \left( s_\theta^0,t \right)$ converges smoothly to the nudged fixed point $s_\theta^\beta$.
As for the left-hand side, when $t$ is large, $L_\theta \left( \s,t \right)$ is close to the cost of the energy minimum, and thus has little sensitivity to the initial state $\s$.

Finally, as $t \to \infty$ in Eq.~\ref{eq:cor-2}, one recovers the gradient formula of Equilibrium Propagation (Theorem \ref{thm:equilibrium-prop}).
Interestingly, Eq.~\ref{eq:cor-2} shows that, in Equilibrium Propagation, \textit{halting the second phase before convergence to the nudged fixed point} corresponds to \textit{Truncated Recurrent Backpropagation}.



\section{Conclusion}

Our work establishes a close connection between two algorithms for fixed point recurrent networks, namely Recurrent Backpropagation and Equilibrium Propagation.
The temporal derivatives of the neural activities in the second phase of Equilibrium Propagation
are equal to the error derivatives which Recurrent Backpropagation computes iteratively.
Moreover, we have shown that halting the second phase before convergence in Equilibrium Propagation is equivalent to Truncated Recurrent Backpropagation.
Our work supports the hypothesis that, in biological networks, temporal changes in neural activities may represent error signals for supervised learning from a machine learning perspective.

One important drawback of the theory presented here is that it assumes the existence of an energy function.
In the case of the Hopfield energy, this implies symmetric connections between neurons.
It would be worth extending the analysis to dynamics which do not involve energy functions but we leave this for future work.
Another concern is the fact that our algorithm is rate-based whereas biological neurons emit spikes.
Ideally we would like a theory applicable to spiking networks.
Finally, the assumption of the existence of specialized 'output neurons' ($s_0$ here) would need to be relaxed too.

From a practical point of view, another issue is that the time needed to converge to the first fixed point was experimentally found to grow exponentially with the number of layers in \citet{Scellier+Bengio-frontiers2017}.
Although Eq.~\ref{eq:cor-2} provides a new justification for saving time by stopping the second phase early, our algorithm (as well as Recurrent Backpropagation) still requires convergence to the free fixed point in the first phase.


\section*{Acknowledgments}

The authors would like to thank 
NSERC, CIFAR, Samsung and Canada Research Chairs for funding.


\bibliographystyle{abbrvnat}
\bibliography{biblio}

\begin{thebibliography}{8}
\providecommand{\natexlab}[1]{#1}
\providecommand{\url}[1]{\texttt{#1}}
\expandafter\ifx\csname urlstyle\endcsname\relax
  \providecommand{\doi}[1]{doi: #1}\else
  \providecommand{\doi}{doi: \begingroup \urlstyle{rm}\Url}\fi

\bibitem[Almeida(1987)]{Almeida87}
L.~B. Almeida.
\newblock A learning rule for asynchronous perceptrons with feedback in a
  combinatorial environment.
\newblock volume~2, pages 609--618, San Diego 1987, 1987. IEEE, New York.

\bibitem[Cohen and Grossberg(1983)]{cohen1983absolute}
M.~A. Cohen and S.~Grossberg.
\newblock Absolute stability of global pattern formation and parallel memory
  storage by competitive neural networks.
\newblock \emph{IEEE transactions on systems, man, and cybernetics}, \penalty0
  (5):\penalty0 815--826, 1983.

\bibitem[Crick(1989)]{crick-nature1989}
F.~Crick.
\newblock The recent excitement about neural networks.
\newblock \emph{Nature}, 337\penalty0 (6203):\penalty0 129--132, 1989.

\bibitem[Hinton and McClelland(1988)]{Hinton+McClelland-1988}
G.~E. Hinton and J.~L. McClelland.
\newblock Learning representations by recirculation.
\newblock In D.~Z. Anderson, editor, \emph{Neural Information Processing
  Systems}, pages 358--366. American Institute of Physics, 1988.

\bibitem[Hopfield(1984)]{Hopfield84}
J.~J. Hopfield.
\newblock Neurons with graded responses have collective computational
  properties like those of two-state neurons.
\newblock 81, 1984.

\bibitem[LeCun et~al.(1988)LeCun, Touresky, Hinton, and
  Sejnowski]{lecun1988theoretical}
Y.~LeCun, D.~Touresky, G.~Hinton, and T.~Sejnowski.
\newblock A theoretical framework for back-propagation.
\newblock In \emph{Proceedings of the 1988 connectionist models summer school},
  pages 21--28. CMU, Pittsburgh, Pa: Morgan Kaufmann, 1988.

\bibitem[Pineda(1987)]{Pineda87}
F.~J. Pineda.
\newblock Generalization of back-propagation to recurrent neural networks.
\newblock 59:\penalty0 2229--2232, 1987.

\bibitem[Scellier and Bengio(2017)]{Scellier+Bengio-frontiers2017}
B.~Scellier and Y.~Bengio.
\newblock Equilibrium propagation: Bridging the gap between energy-based models
  and backpropagation.
\newblock \emph{Frontiers in computational neuroscience}, 11, 2017.

\end{thebibliography}


\newpage
\appendix
\part*{Appendix}


\section{Recurrent Backpropagation - Proof}
\label{sec:rec-backprop-proof}

\begin{proof}[Proof of Theorem \ref{thm:rec-backprop}]
	First of all, by definition of $L$ (Eq.~\ref{eq:projected-cost}) we have $L_\theta(\s,0) = C_\theta(\s)$.
	Therefore the initial conditions (Eq.~\ref{eq:Cauchy-1} and Eq.~\ref{eq:Cauchy-2}) are satisfied:
	\begin{equation}
		\overline{S}_0 = \frac{\partial L_\theta}{\partial s} \left( s_\theta^0,0 \right)
		= \frac{\partial C_\theta}{\partial s} \left( s_\theta^0 \right)
	\end{equation}
	and
	\begin{equation}
		\overline{\Theta}_0 = \frac{\partial L_\theta}{\partial \theta} \left( s_\theta^0,0 \right)
		= \frac{\partial C_\theta}{\partial \theta} \left( s_\theta^0 \right).
	\end{equation}

	It remains to show Eq.~\ref{eq:Cauchy-3} and Eq.~\ref{eq:Cauchy-4}.
	We omit temporarily to write the dependence in $\theta$ to keep notations simple.
	As a preliminary result, we show that for any initial state $\s$ and time $t$ we have
	\footnote{Eq.~\ref{eq:proof1-1} is the Kolmogorov backward equation for deterministic processes.}
	\begin{equation}
		\label{eq:proof1-1}
		\frac{\partial L}{\partial t}(\s,t) + \frac{\partial L}{\partial s}(\s,t) \cdot \frac{\partial E}{\partial s}(\s) = 0.
	\end{equation}
	To this end note that (by definition of $L$ and $S^0$) we have for any $t$ and $u$
	\begin{equation}
		\label{eq:proof1-2}
		L \left( S^0(\s,u),t \right) = L(\s,t+u).
	\end{equation}
	The derivatives of the right-hand side of Eq.~\ref{eq:proof1-2} with respect to $t$ and $u$ are clearly equal:
	\begin{equation}
		\frac{d}{dt} L(\s,t+u) = \frac{d}{du} L(\s,t+u).
	\end{equation}
	Therefore the derivatives of the left-hand side of Eq.~\ref{eq:proof1-2} are equal too:
	\begin{align}
		\hspace*{-0.25cm} \frac{\partial L}{\partial t}\left( S^0(\s,u),t \right) & = \frac{d}{du} L \left( S^0(\s,u),t \right) \\
		\hspace*{-0.25cm} & = - \frac{\partial L}{\partial s}\left( S^0(\s,u),t \right) \cdot \frac{\partial E}{\partial s}\left( S^0(\s,u) \right).
	\end{align}
	Here we have used the differential equation of motion (Eq.~\ref{eq:free-dynamics}).
	Evaluating this expression for $u=0$ we get Eq.~\ref{eq:proof1-1}.

	Now we are ready to show that $\overline{S}_t = \frac{\partial L}{\partial s} \left( s^0,t \right)$ satisfies the differential equation in Eq.~\ref{eq:Cauchy-3}.
	Differentiating Eq.~\ref{eq:proof1-1} with respect to $s$, we get
	\begin{equation}
		\frac{\partial^2 L}{\partial t \partial s}(\s,t) + \frac{\partial^2 L}{\partial s^2}(\s,t) \cdot \frac{\partial E}{\partial s}(\s)
		+ \frac{\partial L}{\partial s}(\s,t) \cdot \frac{\partial^2 E}{\partial s^2}(\s) = 0.
	\end{equation}
	Evaluating this expression at the fixed point $\s = s^0$
	and using the fixed point condition $\frac{\partial E}{\partial s} \left( s^0 \right) = 0$ we get
	\begin{equation}
		\frac{d}{dt} \frac{\partial L}{\partial s} \left( s^0,t \right)
		= - \frac{\partial^2 E}{\partial s^2} \left( s^0 \right) \cdot \frac{\partial L}{\partial s} \left( s^0,t \right).
	\end{equation}
	Therefore $\overline{S}_t = \frac{\partial L}{\partial s} \left( s^0,t \right)$ satisfies Eq.~\ref{eq:Cauchy-3}.

	We prove Eq.~\ref{eq:Cauchy-4} similarly.
	Differentiating Eq.~\ref{eq:proof1-1} with respect to $\theta$, we get
	\begin{align}
		\frac{\partial^2 L_\theta}{\partial t \partial \theta} \left( \s,t \right) + \frac{\partial^2 L_\theta}{\partial s \partial \theta} \left( \s,t \right) \cdot \frac{\partial E_\theta}{\partial s} (\s) & \nonumber \\
		+ \frac{\partial L_\theta}{\partial s}(\s,t) \cdot \frac{\partial^2 E_\theta}{\partial s \partial \theta}(\s) & = 0.
	\end{align}
	Evaluating this expression at the fixed point $\s = s_\theta^0$ we get
	\begin{equation}
		\frac{d}{dt} \frac{\partial L_\theta}{\partial \theta} \left( s_\theta^0,t \right)
		= - \frac{\partial^2 E_\theta}{\partial \theta \partial s} \left( s_\theta^0 \right) \cdot \frac{\partial L_\theta}{\partial s} \left( s_\theta^0,t \right).
	\end{equation}
	Hence the result.
\end{proof}


\section{Equilibrium Propagation - Proof}
\label{sec:equi-prop-proof}

In this Appendix we prove Theorem \ref{thm:equilibrium-prop}.
The same proof was already provided in \citet{Scellier+Bengio-frontiers2017}.

Since the data point $(\x,\y)$ does not play any role, its dependence is omitted in the notations.
We assume that the energy function $E_\theta(s)$ and the cost function $C_\theta(s)$
(and thus the augmented energy function $E_\theta^\beta(s)$) are twice differentiable
and that the conditions of the implicit function theorem are satisfied
so that the fixed point $s_\theta^\beta$ is a continuously differentiable function of $(\theta,\beta)$.

\begin{proof}[Proof of Theorem \ref{thm:equilibrium-prop}]
	Recall that we want to show the "gradient formula"
	\begin{equation}
		\label{eq:grad-for-equi-prop}
		\frac{\partial J}{\partial \theta} \left( \theta \right)
		= \lim_{\beta \to 0} \frac{1}{\beta} \left(
		\frac{\partial E_\theta^\beta}{\partial \theta} \left( s_\theta^\beta \right)
		- \frac{\partial E_\theta^0}{\partial \theta} \left( s_\theta^0 \right)
		\right).
	\end{equation}
	The gradient formula Eq.~\ref{eq:grad-for-equi-prop} is a particular case of the following formula
	\footnote{
	The notations $\frac{\partial E_\theta^\beta}{\partial \theta}$ and $\frac{\partial E_\theta^\beta}{\partial \beta}$
	are used to mean the \textit{partial derivatives} with respect to the arguments of $E_\theta^\beta$,
	whereas $\frac{d}{d \theta}$ and $\frac{d}{d \beta}$ represent the \textit{total derivatives}
	with respect to $\theta$ and $\beta$ respectively
	(which include the differentiation path through $s_\theta^\beta$).
	The total derivative $\frac{d}{d \theta}$ (resp. $\frac{d}{d \beta}$) is performed for fixed $\beta$ (resp. fixed $\theta$).
	},
	when evaluated at the point $\beta=0$:
	\begin{equation}
		\label{lemma-equi-prop}
		\frac{d}{d \theta} \frac{\partial E_\theta^\beta}{\partial \beta} \left( s_\theta^\beta \right)
		= \frac{d}{d \beta} \frac{\partial E_\theta^\beta}{\partial \theta} \left( s_\theta^\beta \right).
	\end{equation}
	Therefore, in order to prove Eq.~\ref{eq:grad-for-equi-prop}, it is sufficient to prove  Eq.~\ref{lemma-equi-prop}.

	First, the cross-derivatives of $(\theta,\beta) \mapsto E_\theta^\beta \left( s_\theta^\beta \right)$ are equal:
	\begin{equation}
		\label{eq:cross-derivatives}
		\frac{d}{d \theta} \frac{d}{d \beta} E_\theta^\beta \left( s_\theta^\beta \right)
		= \frac{d}{d \beta} \frac{d}{d \theta} E_\theta^\beta \left( s_\theta^\beta \right).
	\end{equation}
	Second, by the chain rule of differentiation we have
	\begin{align}
	  \frac{d}{d\beta} E_\theta^\beta \left( s_\theta^\beta \right)
	  & = \frac{\partial E_\theta^\beta}{\partial \beta} \left( s_\theta^\beta \right)
	  + \frac{\partial E_\theta^\beta}{\partial s} \left( s_\theta^\beta \right) \cdot \frac{\partial s_\theta^\beta}{\partial \beta} \\
	  & = \frac{\partial E_\theta^\beta}{\partial \beta} \left( s_\theta^\beta \right).
	  \label{eq:derivative-beta}
	\end{align}
	Here we have used the fixed point condition
	\begin{equation}
		\frac{\partial E_\theta^\beta}{\partial s} \left( s_\theta^\beta \right) = 0.
	\end{equation}
	Similarly we have
	\begin{equation}
	  \label{eq:derivative-theta}
	  \frac{d}{d\theta} E_\theta^\beta \left( s_\theta^\beta \right)
	  = \frac{\partial E_\theta^\beta}{\partial \theta} \left( s_\theta^\beta \right).
	\end{equation}
	Plugging Eq.~\ref{eq:derivative-beta} and Eq.~\ref{eq:derivative-theta} in Eq.~\ref{eq:cross-derivatives}, we get Eq.~\ref{lemma-equi-prop}.
	Hence the result.
\end{proof}


\section{Temporal Derivatives Code For Error Derivatives - Proof}
\label{sec:proofs}

\begin{proof}[Proof of Theorem \ref{thm:error-derivatives}]
	In order to prove Theorem \ref{thm:error-derivatives}, we have to show that the process $(\widetilde{S}_t,\widetilde{\Theta}_t)$
	satisfies the same differential equations as $(\overline{S}_t,\overline{\Theta}_t)$,
	namely Eq.~\ref{eq:Cauchy-1}, Eq.~\ref{eq:Cauchy-2}, Eq.~\ref{eq:Cauchy-3} and Eq.~\ref{eq:Cauchy-4} (Theorem \ref{thm:rec-backprop}).
	We will conclude by using the uniqueness of the solution to the differential equation with initial condition.

	First of all, note that
	\begin{align}
		& \left. \frac{\partial^2 S_\theta^\beta}{\partial \beta \partial t} \right|_{\beta=0} \left( s_\theta^0,t \right) \nonumber \\
		= & \lim_{\beta \to 0} \frac{1}{\beta} \left( \frac{\partial S_\theta^\beta}{\partial t} \left( s_\theta^0,t \right) - \frac{\partial S_\theta^0}{\partial t} \left( s_\theta^0,t \right) \right) \\
		= & \lim_{\beta \to 0} \frac{1}{\beta} \frac{\partial S_\theta^\beta}{\partial t} \left( s_\theta^0,t \right).
	\end{align}
	The latter equality comes from the fact that $S_\theta^0 \left( s_\theta^0,t \right) = s_\theta^0$ for every $t \geq 0$,
	implying that $\frac{\partial S_\theta^0}{\partial t} \left( s_\theta^0,t \right) = 0$ at every moment $t \geq 0$.
	Furthermore
	\begin{align}
		& \left. \frac{d}{d\beta} \right|_{\beta=0} \frac{\partial E_\theta^\beta}{\partial \theta} \left( S_\theta^\beta \left( s_\theta^0,t \right) \right) \nonumber \\
		= & \lim_{\beta \to 0} \frac{1}{\beta} \left(
			\frac{\partial E_\theta^\beta}{\partial \theta} \left( S_\theta^\beta \left( s_\theta^0,t \right) \right)
			- \frac{\partial E_\theta^0}{\partial \theta} \left( S_\theta^0 \left( s_\theta^0,t \right) \right)
			\right) \\
		= & \lim_{\beta \to 0} \frac{1}{\beta} \left(
			\frac{\partial E_\theta^\beta}{\partial \theta} \left( S_\theta^\beta \left( s_\theta^0,t \right) \right)
			- \frac{\partial E_\theta}{\partial \theta} \left( s_\theta^0 \right)
			\right).
	\end{align}
	Again the latter equality comes from the fact that $S_\theta^0 \left( s_\theta^0,t \right) = s_\theta^0$ for every $t \geq 0$.
	Therefore
	\begin{align}
		\widetilde{S}_t      & = - \left. \frac{\partial^2 S_\theta^\beta}{\partial \beta \partial t} \right|_{\beta=0} \left( s_\theta^0,t \right), \qquad  \forall t \geq 0, \\
		\widetilde{\Theta}_t & = \left. \frac{d}{d\beta} \right|_{\beta=0} \frac{\partial E_\theta^\beta}{\partial \theta} \left( S_\theta^\beta \left( s_\theta^0,t \right) \right), \qquad  \forall t \geq 0.
	\end{align}

	Now we prove that $\widetilde{S}_t$ is the solution of Eq.~\ref{eq:Cauchy-1} and Eq.~\ref{eq:Cauchy-3}.
	We omit to write the dependence in $\theta$ to keep notations simple.
	The process $\left( S^\beta \left( s^0,t \right) \right)_{t \geq 0}$ is the solution of the differential equation
	\begin{equation}
		\label{eq:proof3-1}
		\frac{\partial S^\beta}{\partial t} \left( s^0,t \right) = - \frac{\partial E^\beta}{\partial s} \left( S^\beta \left( s^0,t \right) \right).
	\end{equation}
	with initial condition $S^\beta \left( s^0,0 \right) = s^0$.
	Differentiating Eq.~\ref{eq:proof3-1} with respect to $\beta$, we get

	\begin{equation}
		\begin{split}
			\frac{d}{dt} \frac{\partial S^\beta}{\partial \beta} \left( s^0,t \right) =
			& - \frac{\partial^2 E^\beta}{\partial s \partial \beta} \left( S^\beta \left( s^0,t \right) \right) \\
			& - \frac{\partial^2 E^\beta}{\partial s^2} \left( S^\beta \left( s^0,t \right) \right) \cdot \frac{\partial S^\beta}{\partial \beta} \left( s^0,t \right).
		\end{split}
	\end{equation}
	Evaluating at $\beta=0$ and using the fact that $S^0 \left( s^0,t \right) = s^0$, we get

	\begin{equation}
		\label{eq:proof3-2}
		\begin{split}
			\frac{d}{dt} \left. \frac{\partial S^\beta}{\partial \beta} \right|_{\beta=0} \left( s^0,t \right) =
			& - \frac{\partial C}{\partial s} \left( s^0 \right) \\
			& - \frac{\partial^2 E}{\partial s^2} \left( s^0 \right) \cdot
			\left. \frac{\partial S^\beta}{\partial \beta} \right|_{\beta=0} \left( s^0,t \right).
		\end{split}
	\end{equation}
	Since at time $t=0$ the initial state of the network $S^\beta \left( s^0,0 \right) = s^0$ is independent of $\beta$, we have
	\begin{equation}
		\label{eq:proof3-4}
		\frac{\partial S^\beta}{\partial \beta} \left( s^0,0 \right) = 0.
	\end{equation}
	Therefore, evaluating Eq.~\ref{eq:proof3-2} at $t=0$, we get the initial condition (Eq.~\ref{eq:Cauchy-1})
	\begin{equation}
		\widetilde{S}_0
		= - \left. \frac{\partial^2 S^\beta}{\partial t \partial \beta} \right|_{\beta=0} \left( s^0,0 \right)
		= \frac{\partial C}{\partial s} \left( s^0 \right).
	\end{equation}
	Moreover, differentiating Eq.~\ref{eq:proof3-2} with respect to time we get
	\begin{equation}
		\frac{d}{dt} \left. \frac{\partial^2 S^\beta}{\partial t \partial \beta} \right|_{\beta=0} \left( s^0,t \right) =
		- \frac{\partial^2 E}{\partial s^2} \left( s^0 \right) \cdot
		\left. \frac{\partial^2 S^\beta}{\partial t \partial \beta} \right|_{\beta=0} \left( s^0,t \right).
	\end{equation}
	Hence Eq.~\ref{eq:Cauchy-3}:
	\begin{equation}
		\frac{d}{dt} \widetilde{S}_t = - \frac{\partial^2 E}{\partial s^2} \left( s^0 \right) \cdot \widetilde{S}_t.
	\end{equation}

	Now we prove the result for $\widetilde{\Theta}_t$ (Eq.~\ref{eq:Cauchy-2} and Eq.~\ref{eq:Cauchy-4}).
	First we differentiate $\frac{\partial E_\theta^\beta}{\partial \theta} \left( S_\theta^\beta \left( s_\theta^0,t \right) \right)$ with respect to $\beta$:
	\begin{equation}
		\begin{split}
			\frac{d}{d\beta} \frac{\partial E_\theta^\beta}{\partial \theta} \left( S_\theta^\beta \left( s_\theta^0,t \right) \right)
			& = \frac{\partial E_\theta^\beta}{\partial \theta \partial \beta} \left( S_\theta^\beta \left( s_\theta^0,t \right) \right) \\
			& + \frac{\partial E_\theta^\beta}{\partial \theta \partial s}     \left( S_\theta^\beta \left( s_\theta^0,t \right) \right)
			\cdot \frac{\partial S_\theta^\beta}{\partial \beta} \left( s_\theta^0,t \right).
		\end{split}
	\end{equation}
	Again we evaluate at $\beta=0$ and we use the fact that $S_\theta^0 \left( s_\theta^0,t \right) = s_\theta^0$.
	We get

	\begin{equation}
		\label{eq:proof3-3}
		\begin{split}
			\left. \frac{d}{d\beta} \right|_{\beta=0} \frac{\partial E_\theta^\beta}{\partial \theta} \left( S_\theta^\beta \left( s_\theta^0,t \right) \right)
			= & \frac{\partial C_\theta}{\partial \theta} \left( s_\theta^0 \right) \\
			+ & \frac{\partial E_\theta}{\partial \theta \partial s} \left( s_\theta^0 \right)
			\cdot \left. \frac{\partial S_\theta^\beta}{\partial \beta} \right|_{\beta=0} \left( s_\theta^0,t \right).
		\end{split}
	\end{equation}
	Evaluating Eq.~\ref{eq:proof3-3} at time $t=0$ and using Eq.~\ref{eq:proof3-4} we get the initial condition (Eq.~\ref{eq:Cauchy-2})
	\begin{equation}
		\widetilde{\Theta}_0 = \left. \frac{d}{d\beta} \right|_{\beta=0} \frac{\partial E_\theta^\beta}{\partial \theta} \left( S_\theta^\beta \left( s_\theta^0,0 \right) \right) = \frac{\partial C_\theta}{\partial \theta} \left( s_\theta^0 \right).
	\end{equation}
	Moreover, differentiating Eq.~\ref{eq:proof3-3} with respect to time we get
	\begin{equation}
		\frac{d}{dt} \left. \frac{d}{d\beta} \right|_{\beta=0} \frac{\partial E_\theta^\beta}{\partial \theta} \left( S_\theta^\beta \left( s_\theta^0,t \right) \right)
		= \frac{\partial E_\theta}{\partial \theta \partial s} \left( s_\theta^0 \right)
		\cdot \left. \frac{\partial^2 S_\theta^\beta}{\partial t \partial \beta} \right|_{\beta=0} \left( s_\theta^0,t \right).
	\end{equation}
	Hence Eq.~\ref{eq:Cauchy-4}:
	\begin{equation}
		\frac{d}{dt} \widetilde{\Theta}_t
		= - \frac{\partial E_\theta}{\partial \theta \partial s} \left( s_\theta^0 \right)
		\cdot \widetilde{S}_t.
	\end{equation}
	This completes the proof.
\end{proof}

\end{document}